\documentclass{article}
\usepackage{floatrow}
\newfloatcommand{capbtabbox}{table}[][\FBwidth]
\usepackage{blindtext}

\newcommand{\rebut}[1]{{\color{black} #1}}

\usepackage{PRIMEarxiv}
\usepackage{fancyhdr}       
\usepackage{graphicx}       
\usepackage[utf8]{inputenc} 
\usepackage[T1]{fontenc}    
\usepackage{hyperref}       
\usepackage{url}            
\usepackage{booktabs}       
\usepackage{nicefrac}       
\usepackage{microtype}      
\usepackage[dvipsnames,svgnames,x11names,hyperref, enumerate]{xcolor}
\definecolor{blueish}{HTML}{007F99}
\definecolor{purpleish}{HTML}{72177A}
\definecolor{grape}{HTML}{6f2da8}
\definecolor{magenta}{HTML}{c229a9}
\hypersetup{
	colorlinks,
	citecolor=purpleish,
	filecolor=red,
	linkcolor=blueish,
	urlcolor=Blue
}

\usepackage{amsmath}
\usepackage{amssymb}
\usepackage{mathtools}
\usepackage{amsthm}
\usepackage{caption}
\usepackage{subcaption}

\usepackage[capitalize,noabbrev]{cleveref}

\theoremstyle{plain}
\newtheorem{theorem}{Theorem}[section]

\newtheorem{lemma}[theorem]{Lemma}

\theoremstyle{definition}

\theoremstyle{remark}

\usepackage[textsize=tiny]{todonotes}

\DeclareMathOperator*{\E}{\mathbb{E}}

\newcommand{\sgn}{\mathop{\mathrm{sgn}}}
\newcommand{\mypar}[1]{\noindent \textbf{#1}}
\newcommand{\DSM}{\mathsf{DSM}}
\newcommand{\LSH}{\mathsf{LSH}}
\newcommand{\eps}{\epsilon}

\author{%
  Cenk Baykal \\
  Google Research \\
  \texttt{baykalc@google.com} \\
  \And
  Nishanth Dikkala \\
  Google Research \\
  \texttt{nishanthd@google.com} \\
  \AND
  Rina Panigrahy \\
  Google Research \\
  \texttt{rinap@google.com} \\
  \And
  Cyrus Rashtchian \\
  Google Research \\
  \texttt{cyroid@google.com} \\
  \And
  Xin Wang \\
  Google Research \\
  \texttt{wanxin@google.com} \\
}
\title{A Theoretical View on Sparsely Activated Networks}

\begin{document}

\maketitle

\begin{abstract}
    Deep and wide neural networks successfully fit very complex functions today, but dense models are starting to be prohibitively expensive for inference. To mitigate this, one promising direction is networks that activate a sparse subgraph of the network. The subgraph is chosen by a data-dependent routing function, enforcing a fixed mapping of inputs to subnetworks (e.g., the Mixture of Experts (MoE) paradigm in Switch Transformers). 
    \rebut{However, prior work is largely empirical, and while existing routing functions work well in practice, they do not lead to theoretical guarantees on approximation ability.}
    \rebut{We aim to provide a theoretical explanation for the power of sparse networks.}
    As our first contribution, we present a formal model of data-dependent sparse networks that captures salient aspects of popular architectures.
    \rebut{We then introduce a routing function based on locality sensitive hashing (LSH) that enables us to reason about how well sparse networks approximate target functions.}
    \rebut{After representing LSH-based sparse networks with our model, we prove that sparse networks can match the approximation power of dense networks on Lipschitz functions.} 
    \rebut{Applying LSH on the input vectors means that the experts interpolate the target function in different subregions of the input space.}
    \rebut{To support our theory, we define various datasets based on Lipschitz target functions, and we show that sparse networks give a favorable trade-off between number of active units and approximation quality.}
\end{abstract}
\section{Introduction}
\label{sec:intro}

Overparameterized neural networks continue to yield performance gains as their sizes increase. This trend has been most prominent with large Transformer based language models \cite{brown2020language, devlin2018bert,raffel2019exploring}. However, using large, dense networks makes training and inference very expensive, and computing a forward pass may require trillions of floating point operations (FLOPs). It is an active area of research to improve the scalability and efficiency without decreasing the expressiveness or quality of the models. 

One way to achieve this goal is to only activate part of the network at a time. For example, the Mixture of Experts (MoE) paradigm~\cite{jordan1994hierarchical,shazeer2017outrageously} uses a two-step approach. First, each input is mapped to a certain subnetwork, known as an expert. Then, upon receiving this input, only this particular subnetwork performs inference, leading to a smaller number of operations compared to the total number of parameters across all experts. Some variations of this idea map parts of an input (called tokens) individually to different experts. Another well-studied generalization is to map an input to a group of subnetworks and aggregate their output in a weighted manner. Switch Transformers~\cite{fedus2021switch} successfully use a refined version of the MoE idea, where the input may be the embedding of a token or part of a hidden layer's output. Researchers have investigated many ways to perform the mapping, such as Scaling Transformers~\cite{jaszczur2021sparse} or using pseudo-random hash functions~\cite{roller2021hash}. In all cases, sparsity occurs in the network because after the input is mapped to the subnetwork, only this subset of neurons needs to be activated. Moreover, the computation of the mapping function, sometimes referred to as a `routing' function, takes significantly less time than the computation across all experts.

The success of these approaches is surprising. Restricting to a subnetwork could instead {\em reduce} the expressive power and the quality of the model. However, the guiding wisdom seems to be that it suffices to maximize the number of parameters while ensuring computational efficiency. Not all parameters of a network are required for the model to make its prediction for any given example. 
\rebut{Our goal is to provide a theoretical explanation for the power of sparse models on concrete examples.}



\paragraph{Our Results.}
Our first contribution is a formal model of networks that have one or more sparsely activated layers. These sparse layers are data dependent: the subset of network activations and their weight coefficients depend on the input. We formally show that our model captures popular architectures (e.g., Switch and Scaling Transformers) by simulating the sparse layers in these models.

We next prove that \rebut{the above class of sparsely activated models} can learn a large class of functions \rebut{more efficiently than dense models}.
\rebut{As our main proof technique, we introduce a new routing function, based on locality sensitive hashing (LSH), that allows us to theoretically analyze sparsely activated layers in a network.}
LSH maps points in a metric space (e.g., $\mathbb{R}^d$) to `buckets' such that nearby points map to same bucket. The total number of buckets used by a LSH function is referred to as the size of the hash table.
\rebut{Prior work has already identified many routing functions that work well in practice. We do not aim to compete with these highly-optimized methods, but rather we aim to use LSH-based routing as a way to reason about the approximation ability of sparse networks.}

In Theorem~\ref{thm:main-lipschitz-lsh-ub}, we show that LSH-based sparse models can approximate real-valued Lipschitz functions in $\mathbb{R}^d$. Although real data often lives in high dimensions, the underlying manifold where the inputs are supported is often assumed to be low-dimensional (e.g. the manifold of natural images vs $\mathbb{R}^{224 \times 224}$ for a $224\times 224$ image). We model this by assuming that our inputs lie in a $k$-dimensional manifold within $\mathbb{R}^d$. To get $\epsilon$ approximation error, we need an LSH table of size approximately $O((\sqrt{dk}/\epsilon)^k)$ but a forward pass only requires time $O(dk\log(1/\epsilon))$ as only one of the $O((\sqrt{dk}/\epsilon)^k)$ non-empty buckets are accessed for any given example.

In Theorem~\ref{thm:dense-lb}, we complement our upper bounds by proving a lower bound of $\Omega((2\sqrt{d}/\epsilon)^k)$ on the size needed for both dense and sparse models (when points live in $k$ dimensions). This also transforms into a lower bound on the number of samples required to learn these functions using a dense model. This lower bound implies that a forward pass on a dense model takes time $\Omega(d(2/\sqrt{k}\epsilon)^k)$, which is exponentially worse than the time taken by the sparse model.
Altogether, we show that for the general class of Lipschitz functions, sparsely activated layers are as expressive as dense layers, while performing significantly fewer floating point operations (FLOPs) per example.


\rebut{To support our theory, we perform experiments in Section~\ref{sec:experiments} on approximating Lipschitz functions with sparse networks.}
\rebut{We identify several synthetic datasets where models with data-dependent sparse layers outperform dense models of the same size. Moreover, we achieve these results with relatively small networks. This provides a few minimum working examples of when sparsity is beneficial. 
Our results complement the existing work that has already shown the success of very large MoE models on real-world NLP and vision tasks~\cite{fedus2021switch, jaszczur2021sparse, riquelme2021scaling}.}
Our experiments also generalize the intuition from Figure \ref{fig:1d_lsh}, where the model outputs a different constant value for each bucket.
A puzzling empirical result was observed in \cite{roller2021hash} where a pseudo-random uniform hash function is used to route input tokens to experts. Intuitively, this might map very nearby or similar points to distinct experts, leading to a potentially non-smooth behavior. 
\rebut{We show that on our synthetic, Lipschitz datasets, the $\LSH$ model outperforms uniform random hashing, which further corroborates our theory.}


\begin{figure}[h!]
  \centering
  \begin{subfigure}[b]{0.5\linewidth}
    \centering
    \includegraphics[width=\textwidth]{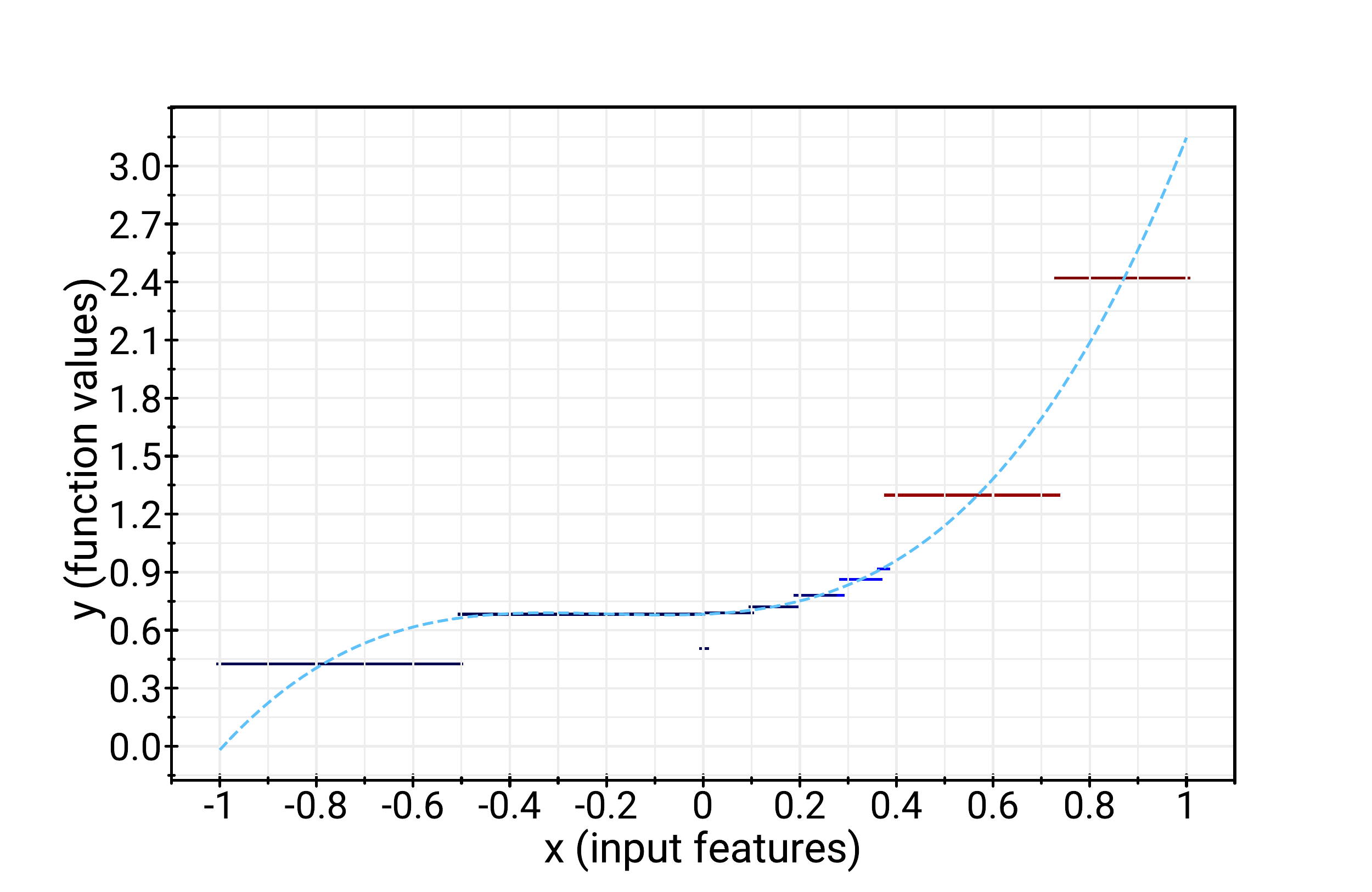}
    \caption{1D example}
    \label{fig:1d_lsh}
  \end{subfigure}%
  \begin{subfigure}[b]{0.5\linewidth}
    \centering
    \includegraphics[width=\textwidth]{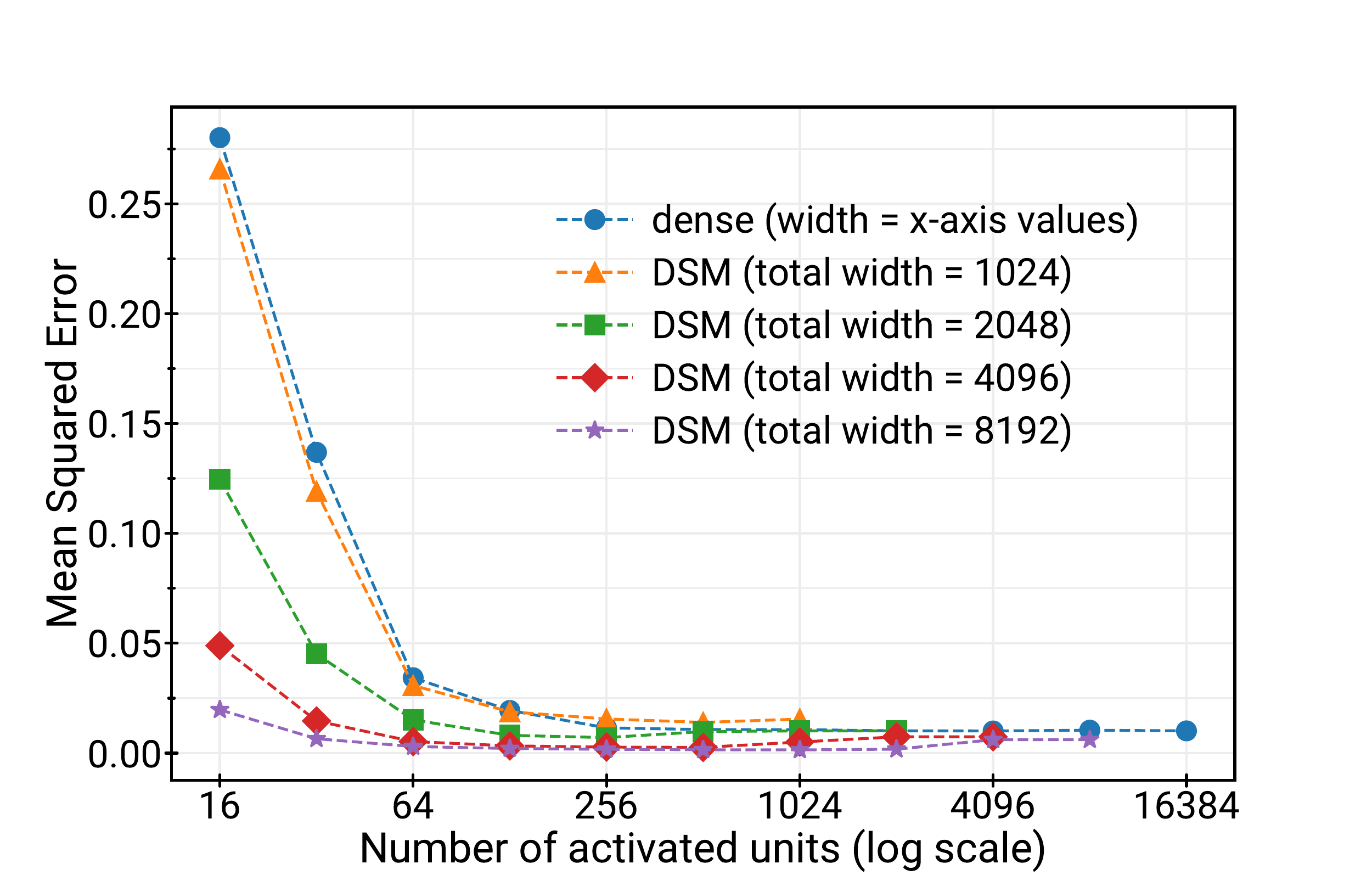}
    \caption{sparse/dense scaling behavior}
    \label{fig:sparse_vs_dense}
  \end{subfigure}
  \vspace{-0.2in}
  \caption{Learning polynomial functions with data-dependent sparse models: (a) Dashed curve is the target function graph, piecewise constant curve is the learned LSH model output, and different LSH buckets are indicated by different colors. Note neighboring points are hashed to the same bucket, where a simple constant can be learned to approximate target function values; (b) Scaling behavior of data-dependent sparse models (DSM) versus dense models, in which sparse models outperform dense models of the same size, see Section \ref{sec:prelims} and Section \ref{sec:experiments} for details.}
  \label{fig:intro_expt}
\end{figure}

\paragraph{Related Work.}
\label{sec:relatedwork}
Sparsely activated networks have had enormous empirical success~\cite{allingham2021sparse, artetxe2021efficient,du2021glam, kim2020fastformers, nie2021dense, riquelme2021scaling, shazeer2017outrageously, wang2021ernie}. 
The Switch Transformer~\cite{fedus2021switch} is one of the first major, contemporary applications of the sparsely activated layers. Follow-up works such as Scaling Transformers~\cite{jaszczur2021sparse} and other hash functions~\cite{roller2021hash} aim to improve the sparse layers. These papers build upon seminal MoE works~\cite{jacobs1995methods, jacobs1991adaptive, jordan1994hierarchical}, and other uses of the MoE paradigm~\cite{lepikhin2021gshard,shazeer2018mesh-tensorflow}.  
There has been work on using statistical mechanics to analyze the generalization of MoE models~\cite{kang1996statistical}. However, to the best of our knowledge, there is no systematic theoretical study of modern sparsely activated networks. 

The above work on dynamic sparsity builds on previous \emph{static} sparsity efforts, e.g., weight quantization~\cite{li2020train}, dropout~\cite{srivastava2014dropout}, and pruning (see the survey~\cite{hoefler2021sparsity} and references). Static sparsity means that the subnetwork activation does not change in a data-dependent way. The focus is on generalization and compression, instead of achieving fast inference time with a huge number of parameters. 

%

Our work builds on locality sensitive hashing (LSH), a well-studied technique for approximate nearest neighbor search (see the survey~\cite{andoni2018approximate} or the book~\cite{har2011geometric} and references therein for LSH background). A seminal and popular LSH family uses random hyperplanes to partition the space~\cite{andoni2015practical, charikar2002similarity}. Another option is to use balls centered at lattice points, which works well for $\ell_2$ distance~\cite{andoni2006near}. We use these ideas to design sparse networks. For uses of LSH in deep learning, sketch-based memory improves network capacity~\cite{ghazi2019recursive, panigrahy2021sketch}. Other empirical work uses LSH to improve dense network training time or memory~\cite{chen2020mongoose, chen2015compressing, rae2016scaling}. Our work differs from the prior studies because we implement the LSH-based approach with sparsely activated networks, with the goal of reducing inference time and achieving low regression error. LSH-based sparsity may be most relevant for vision transformers~\cite{khan2021transformers, riquelme2021scaling, wu2020visual}, using distance metrics and routing functions tailored to images, or for time series analysis~\cite{wen2022transformers}.

\section{Preliminaries}
\label{sec:prelims}
Let $f:\mathbb{R}^d \to \mathbb{R}$ be a multivariate real-valued function that we want to learn with a neural network. We consider regression, and we aim to minimize the mean-squared error or the $\ell_{\infty}$ error.
For a function $f$ defined on a set $\Gamma$ and an estimator $\hat{f}_n$, we define
$\|f - \hat{f}_n\|_{\infty} = \sup_{x \in \Gamma} |f(x) - \hat{f}_n(x)|.$
For some results, we approximate $f$ on a subset $\mathcal{V} \subseteq \mathbb{R}^d$. For example, $\mathcal{V}$ may be the intersection of $[-1,1]^d$ and a $k$-dimensional subspace. We also consider the case when $f$ is $L$-Lipschitz, meaning that $|f(x) - f(x')| \leq L\cdot \|x-x'\|_2$ for all $x, x' \in \mathbb{R}^d$. We define $[n] = \{1,2,\ldots, n\}$.

\subsection{Data-Dependent Sparse Model}

A dense neural network $g$ with a fully connected final layer can be expressed as $g(x) = A \cdot \phi(x)$ where $A \in \mathbb{R}^{1 \times t}$ is a matrix and $\phi: \mathbb{R}^{d} \to \mathbb{R}^t$ is a function (e.g., $\phi$ captures the representation learned up until the final layer). Here, $t$ is the width of the final layer and $d$ is the input dimensionality.

Our focus is on the types of networks captured by what we call the \emph{Data-Dependent Sparse Model ($\DSM$)}. 
This is a network with a sparsely activated final layer. Formally, let $t$ be the width, and let $s \leq t$ be a sparsity parameter. Then, we consider functions $g$ of the form
$g(x) = A^x \cdot \phi(x)$ where 
$A^x \in \mathbb{R}^{1 \times t}$ and $\phi: \mathbb{R}^{d} \to \mathbb{R}^t$. The crux of the model is the final layer. The sparsity comes from letting $A^x = A \circ \mathsf{mask}(x)$, where $\mathsf{mask}(x) \in \{0,1\}^{1 \times t}$ is an $s$-sparse indicator vector, and  ``$\circ$'' is the entry-wise product. The mask zeroes out certain positions, and $A$ contains the learned weights but no longer depends on~$x$. 
Intuitively, the $\mathsf{mask}$ is the ``routing function'' for the sparse activations. To do so, we define  $\mathsf{mask}(x) \circ \phi(x)$ to mean that we zero out element of the vector $\phi(x)$. Then, we have 
$$g(x) = (A \circ \mathsf{mask}(x)) \phi(x)  = A (\mathsf{mask}(x) \circ \phi(x)).$$ 
Under the above definitions, let $\DSM(d,s,t)$ be the set of functions $g = (A \circ \mathsf{mask}(x)) \cdot \phi(x)$. In what follows, we use $A^x$ as shorthand for $A \circ \mathsf{mask}(x)$.


In the $\DSM$ model, we wish to understand the effect of sparsity on how well the network can approximate certain functions. It is important to allow the final layer $A^x \in \mathbb{R}^{1 \times t}$ to be data-dependent. The values in $A^x$ depend on the target function $f$, and we learn this from training data. 

Real-world transformers may have multiple sparse layers. To capture this, we can compose $\DSM$ functions. For example, two sparse layers comes from $g(x) = A \circ \mathsf{mask}_2(x) \circ \phi_2(\mathsf{mask}_1(x) \circ \phi_1(x))$.  For simplicity, we focus on a single sparse layer in what follows, which suffices for our main results.

\subsection{Hash-based Routing}

Prior sparse models compute hash functions of the input vector $x$ to determine $\mathsf{mask}$ and the network activations~\cite{roller2021hash}. 
Our main theorem uses this strategy, considering LSH families. LSH has the property that nearby points are more likely to end up in the same hash bucket than far away points. We can use LSH to define a general class of efficient regression functions. In Section~\ref{sec:model-results}, we prove that the $\DSM$ model captures popular sparse architectures and the following $\LSH$ model.

\mypar{$\LSH$ Model.} We review an $\LSH$ framework based on sign patterns. Let $h_1,\ldots, h_m:\mathbb{R}^d \to \{-1,1\}$ be $m$ distinct hash functions. Partition the space into $2^m$ buckets based on the $m$ sign patterns $z_x = (h_1(x),\ldots,h_m(x))$ over all $x \in \mathbb{R}^d$. 
Then, for each $z \in \{-1,1\}^m$, we can specify a function $\hat g_{z}(x)$, where the goal is for $g_z$ to approximate the target function $f$ in the part of space associated with $z$ (i.e., points $x$ that have pattern $z$ under $h_1,\ldots, h_m$). More generally, we can allow $s$ sets of such hash functions $(h^i_1,\ldots, h^i_m)$, and $s$ sets of these approximation functions $(\hat g^1_{z^1},\ldots, \hat g^s_{z^s})$ for $i = 1,\ldots, s$. On input $x$, we compute the sign patterns $z^1, \ldots, z^s$ and output $g(x) = \sum_{i=1}^s \alpha_i \hat g^i_{z^i}(x)$. Further, we can restrict each $\hat g^i_z$ to be a degree $\Delta \geq 0$ polynomial. For a fixed LSH family of possible hash functions, we let $\LSH(d, s, m, \Delta)$ denote this class of functions. In many cases, we only need $\hat g^i_z$ to be a constant function, i.e., $\Delta=0$, and we shorten this as $\LSH(d, s, m, 0) := \LSH(d, s, m)$. Our goal is to understand the trade-offs of the ``memory'' parameter $m$ and the ``sparsity'' amount $s$ that we need to approximate large classes of $d$-dimensional functions (i.e., $L$-Lipschitz functions supported on a $k$-dimensional manifold).

\mypar{Euclidean $\LSH$ Model \cite{datar2004locality}.} This is a popular LSH family for points in $\mathbb{R}^d$. In this case, each hash function outputs an integer (instead of $\pm 1$ above). Each bucket in this model is defined by a set of hyperplane inequalities. There are two parameters $(D,\epsilon)$. We sample $D$ random directions $a_1, \ldots, a_D$ where each coordinate of each $a_i$ is an independent normal variable. In addition we sample $b_i \sim \mathrm{Unif}[0,\epsilon]$ independently for $i \in [D]$.  For a point $x \in \mathbb{R}^d$, we compute an index into a bucket via a function $h_i : \mathbb{R}^d \to \mathbb{Z}$ defined as
$$
    h_i(x) = \left\lfloor \frac{a_i^\top x + b_i}{\epsilon} \right\rfloor. 
$$ Here, the index $i$ ranges over $i \in [D]$, leading to a vector of $D$ integers. 




\section{Simulating Models with DSM}
\label{sec:model-results}
We formally justify the $\DSM(d,s,t)$ model by simulating other models using it.
We start with simple examples (interpolation, $k$ nearest neighbor ($k$-NN) regression, and Top-$K$), then move on to transformers and the LSH model. For the simple examples, we only need one hidden layer, where $\phi(x) = \sigma(Bx)$ for a matrix $B \in \mathbb{R}^{t\times d}$ and non-linearity $\sigma$.

\mypar{Interpolation.} We show how to compute $f$ at $t$ points $x_1,\ldots,x_t$. We only need sparsity $s = 1$ and \rebut{$\sigma$ is simply the identity function.}
We set $A_i = f(x_i) / \langle b_i, x_i \rangle$, where $b_i$ is the $i$th row of $B$.
Further, we let $\mathsf{mask}(x_i)$ have a one in the $i$th position and zeroes elsewhere. Then, we define 
$g(x_i) = (A \circ \mathsf{mask}(x_i))Bx_i = f(x_i)$, which computes $f$ at the $t$ input points. 

\mypar{Sparse networks perform $k$-NN regression}. 
We sketch how the $\DSM(d,k,n)$ model can simulate $k$-NN with $g(x) = (A\circ \mathsf{mask}(x)) \sigma(B x)$, when the dataset contains vectors with unit $\ell_2$ norm.Let the rows of $B$ be a set of $n$ unit vectors $b_1,\ldots,b_n \in \mathbb{R}^d$. For the target function~$f$, let $A = \frac{1}{k}(f(b_1),\ldots,f(b_n))$.
Define $\sigma(B x)$ to have ones in the top $k$ largest values in $B x$ and zeroes elsewhere. For a unit vector $x$, these $k$ positions correspond to the $k$ largest inner products $\langle x, b_i \rangle$. Since 
$\|x - b_i\|_2 = 2 - 2\langle x, b_i \rangle$,  the non-zero positions in $\sigma(B x)$ encode the $k$ nearest neighbors of $x$ in $\{b_1,\ldots,b_n\}$.
Thus, $g(x)$ computes the average of $f$ at these $k$ points, which is exactly $k$-NN regression. Moreover, only $k$ entries of $A$ are used for any input, since $\sigma(Bx)$ is $k$-sparse; however, computing $\sigma(Bx)$ takes $O(nd)$ time.
While there is no computational advantage from the sparsity in this case, the fact that $\DSM$ can simulate $k$-NN indicates the power of the model.

\mypar{Simulating Top-$K$ routing.} For an integer $K \in \mathbb{N}^+$, certain MoE models choose the $K$ largest activations to enforce sparsity; in practice, $K \in \{1,2\}$ is common~\cite{riquelme2021scaling}. We use $\mathsf{mask}(x)$ with $\|\mathsf{mask}(x)\|_0 = K$ as the indicator for the $K$ largest values of $Bx$. Then, $\sigma$ is the identity, and $g(x) = (A\circ \mathsf{mask}(x)) B x$ applies $A$ to the top $K$ entries of $Bx$.

\subsection{Simulating transformer models with $\DSM$}
Real-world networks have many hidden layers and multiple sparse layers. For concreteness, we describe how to simulate a sparsely activated final layer. As mentioned above, we can compose functions in the $\DSM$ model to simulate multiple sparse layers.

\mypar{Switch Transformers~\cite{fedus2021switch}.} The sparse activations in transformers depend on
a \emph{routing function} $R: \mathbb{R}^d \to \{1,\ldots, L\}$, where $R(x)$ specifies the subnetwork that is activated (for work on the best choice of $R$, see e.g.,~\cite{roller2021hash}). To put this under the $\DSM(d,s,t)$ model, consider a set of trainable matrices $A_1,\ldots, A_L \in \mathbb{R}^{1 \times s}$, where the total width is $t = s \cdot L$. On input $x$, we think of $A^x$ as a $1 \times t$ matrix with $s$ non-zero entries equal to $A_{R(x)}$. In other words, $A$ is the concatenation of $A_1,\ldots, A_L$, and $\mathsf{mask}(x)$ is non-zero on the positions corresponding to $A_{R(x)}$.

\mypar{Scaling Transformers~\cite{jaszczur2021sparse}.} 
The key difference between Switch and Scaling transformers is that the latter imposes a block structure on the sparsity pattern. Let $t$ be the width, and let $s$ be the number of blocks (each of size $t' = t/s$). Scaling Transformers use only one activation in each of the $s$ blocks. In the $\DSM(d,s,t)$ model, we capture this with $A^x$ as follows.  Let $e_i \in \{0,1\}^{t'}$ denote the standard basis vector (i.e., one-hot encoding of $i \in [t'])$. The sparsity pattern is specified by indices $(i_1,\ldots, i_s)$. Then, $A^x = (\alpha_1 e_{i_1},\ldots, \alpha_b e_{i_s})$ for scalars $\alpha_1,\ldots, \alpha_s \in \mathbb{R}$.

\subsection{Simulating the $\LSH$ model using $\DSM$}

We explain how to represent the $\LSH$ model using the $\DSM$ model. The key aspect is that $A^x$ depends on the LSH buckets that contain $x$, where we have $s$ non-zero weights for the $s$ buckets that contain each input. This resembles the Scaling Transformer, where we use the LSH buckets to define routing function.
In the $\LSH(d, s, m, \Delta)$ model, there are $s$ sets of $m$ hash functions, leading to $s \cdot 2^m$ hash buckets. We use width $t = s \cdot 2^m$ for the $\DSM(d,s,t)$ network. The entries of $A^x$ are in one-to-one mapping with the buckets, where only $s$ entries will be non-zero depending on the $s$ buckets that $x$ hashes to, that is, the values $(h^i_1(x),\ldots, h^i_m(x)) \in \{-1,1\}^m$ for $i = 1,2,\ldots, s$.

We now determine the values of these $s$ non-zero entries. We store a degree $\Delta$ polynomial $\hat g(x) :\mathbb{R}^d \to \mathbb{R}$ associated with each bucket. For our upper bounds, we only need degree $\Delta=0$, but we mention the general case for completeness. If $\Delta=0$, then $\hat g$ is simply a constant $\alpha$ depending on the bucket. 
An input $x$ hashes to $s$ buckets, associated with $s$ scalars $(\alpha_1,\ldots, \alpha_s)$. To form $A^x$, set $s$ entries to the $\alpha_i$ values, with positions corresponding to the buckets. For degree $\Delta \geq 1$, we store coefficients of the polynomials $\hat g$, leading to more parameters. Section~\ref{sec:sparse-random-functions} describes how to use LSH-based sparse networks to approximate Lipschitz functions.

\mypar{Computing and storing the LSH buckets.} Determining the non-zero positions in $A^x$ only requires $O(sm)$ hash computations, each taking $O(d)$ time with standard LSH families (e.g., hyperplane LSH). 
We often take $m$ to be a large constant. Thus, the total number of operations to compute a forward pass in the network $O(smd) \approx O(sd)$.  The variable $m$ above determines the total number of distinct buckets we will have ($2^m$). For an $n$ point dataset, $m = O(\log n)$ is a realistic setting in theory. Therefore, $2^m = \mathrm{poly}(n)$ is often a reasonable size for the hash table and moreover the dense model requires width $O(2^m)$ as well. In practice, the number of experts is usually 32--128~\cite{fedus2021switch, riquelme2021scaling}.
The hash function typically adds very few parameters. For example, to hash $d$ dimensional data into $2^m$ buckets, LSH requires $O(md)$ bits and takes $O(md)$ time as well. In summary, the LSH computation does not asymptotically increase the number of FLOPs for a forward pass in the network. We also later bound the number of non-empty LSH buckets. The memory to store the hash table is comparable to the number of samples required to learn the target function.



\section{Data-Dependent Sparse Models are more Efficient than Dense Models}
\label{sec:sparse-random-functions}

For a general class of Lipschitz functions, LSH-based learners yield similar $\ell_{\infty}$ error as dense neural networks while making inference significantly more efficient.
It is a common belief in the machine learning community that although many of the datasets we encounter can appear to live in high-dimensional spaces, there is a low-dimensional manifold that contains the inputs. To model this, we assume in our theory that the inputs lie in a $k$-dimensional subspace (a linear manifold) of $\mathbb{R}^d$. Here $k \ll d$.
Theorem ~\ref{thm:main-lipschitz-lsh-ub} shows that the LSH model we propose can learn high-dimensional Lipschitz functions with a low $\ell_{\infty}$ error efficiently when the input comes from a uniform distribution on an unknown low-dimensional subspace. Theorem~\ref{thm:main-lipschitz-ub-manifolds} extends this result to when the input comes from an unknown manifold with a bounded curvature. We present Theorem~\ref{thm:main-lipschitz-lsh-ub} here and defer Theorem~\ref{thm:main-lipschitz-ub-manifolds} to Section~\ref{sec:manifolds-apdx}. All proofs are in the appendix.

\begin{theorem} \label{thm:main-lipschitz-lsh-ub}
For any $f: [-1, 1]^d \to \mathbb{R}$ that is $L$-Lipschitz, and for an input distribution $\mathcal{D}$ that is uniform on a $k$-dimensional subspace in $[-1,1]^d$, an LSH-based learner can learn $f$ to $\epsilon$-uniform error with $O(kL\sqrt{d}^k\log(L\sqrt{d}/\epsilon)/\epsilon^k)$ samples using a hash table of size $O(L\sqrt{d}^k/\epsilon^k)$ with probability $\ge 0.8$.
The total time required for a forward pass on a new test sample is $O(dk\log(L\sqrt{d}/\epsilon))$.
\end{theorem}

The key idea behind this theorem is to use LSH to produce a good routing function. The locality of the points hashed to an LSH bucket gives us a way to control the approximation error. By using a large number of buckets, we can ensure their volume is small. Then, outputting a representative value suffices to locally approximate the target Lipschitz function (since its value changes slowly). The next lemma bounds the size of the sub-regions corresponding to each bucket, as well as bounding the number of non-empty buckets.




\begin{lemma}\label{lem:lsh-buckets-size}
Consider a Euclidean $\LSH$ model in $d$ dimensions with $Ck$ hyperplanes and width parameter $\epsilon$ where $C$ is a large enough constant. Consider a region $\Gamma$ defined by the intersection of a $k$-dimensional subspace with $[-1,1]^d$. We have that the $\LSH$ model defines a partitioning of $\Gamma$ into buckets. Let  $c$ be a constant. Then, with probability $\ge 0.9$, 
\begin{enumerate}
    \item Projecting any bucket of the LSH onto $\Gamma$ corresponds to a sub-region with diameter $\le \epsilon/c$.
    \item At most $\left(\frac{2\sqrt{d}}{\epsilon}\right)^{O(k)}$ buckets have a non-empty intersection with $\Gamma$.
\end{enumerate}
\end{lemma}

The above construction assumes knowledge of the dimensionality of the input subspace $k$. Fortunately, any upper bound on $k$ would also suffice. 
The table size of the LSH model scales exponentially in $k$ but not $d$. Thus, an LSH-based learner adapts to the dimensionality of the input subspace.

Theorem~\ref{thm:main-lipschitz-lsh-ub} shows that sparsely activated models (using LSH of a certain table size) are powerful enough to approximate and learn Lipschitz functions. Next, we show a complementary nearly matching lower bound on the width required by dense model to approximate the same class of functions. We use an information theoretic argument.
\begin{theorem} \label{thm:dense-lb}
Consider the problem of learning $L$-Lipschitz functions on $[-1,1]^d$ to $\ell_{\infty}$ error $\epsilon$ when the inputs to the function are sampled from a uniform distribution over an unknown $k$-dimensional subspace of $\mathbb{R}^d \cap [-1,1]^d$.
A dense model of width $t$ with a random bottom layer requires 
$$t = \Omega\left(\frac{(\sqrt{d}L)^k}{(C\epsilon)^k}\right),$$
for a constant $C > 0$ to learn the above class of functions. The number of samples required is
$$\Omega\left(\frac{k(\sqrt{d}L)^k\log(2\sqrt{d}L/C\epsilon)}{(C\epsilon)^k}\right).$$
\end{theorem}

Our approach for this theorem is to use a counting argument. We first bound the number of distinct functions, which is exponential in the number of parameters (measured in bits). We then construct a large family of target functions that are pairwise far apart from each other. Hence, if we learn the wrong function, we incur a large error. Our function class must be large enough to represent any possible target function, and this gives a lower bound on the size of the approximating network.

Theorem~\ref{thm:dense-lb} shows a large gap in the time complexity of inference using a dense model and an LSH model. The inference times taken by a dense model vs.~a sparse model differ exponentially in $1/\epsilon$.
\begin{align}
    \textbf{Sparse: } O\left(dk\log(1/\epsilon) \right)  \qquad \text{vs.} \qquad \textbf{Dense: } \Omega\left(d\left(\frac{2\sqrt{d}}{C\epsilon} \right)^k \right)
\end{align}

Overall, the above theorems show that LSH-based sparsely activated networks can approximate Lipschitz functions on a $k$-dimensional subspace. The size and sample complexity match between sparse and dense models, but the sparse models are exponentially more efficient for inference.

\section{Experiments}
\label{sec:experiments}

To empirically verify our theoretical findings, we align our experiments with our proposed models and compare dense models, data-dependent sparse models ($\DSM$), LSH models, and sparse models with random hash based sparse layers. While $\DSM$ and LSH models are analyzed in the previous section, random hash based layers introduce sparsity with a random hash function, as an alternative of learnable routing modules and LSH in MoE models \cite{roller2021hash}. 
Our goal is to show that the $\DSM$ and $\LSH$ models achieve small MSE while using much fewer activated units than dense models. We also report similar observations on the CIFAR-10 dataset.

\subsection{Experimental set-up} 

\mypar{Models.} We compare dense models and three sparse models that cover a wide range of possible routing methods (Top-$K$ $\DSM$, $\LSH$, and random). We use Top-$K$ routing as a representative $\DSM$ model; it uses the $K$ largest activations, which depends on the input (see Section~\ref{sec:model-results}).
Dense, $\DSM$ and random hash sparse models contain a random-initialized, non-trainable bottom layer. 
For the random hash model, we uniformly choose a subset of the units.
All networks have a trainable top layer. $\LSH$ models have non-trainable hyperplane coefficients for hashing and a trainable scalar in each bucket (this is the network output for inputs in that bucket). For ablations, we vary the number of hidden units and sparsity levels.


\mypar{Data.} 
Our main goal is to empirically evaluate the construction from Theorem~\ref{thm:main-lipschitz-lsh-ub}, while comparing routing functions. To align with our theory, we consider Lipschitz target functions. We synthetically construct the data with two functions that are common bases for continuous functions: random low-degree polynomials and random hypercube functions. Both are normalized to be $L$-Lipscthiz for $L = O(1)$. We also present results on a real dataset, CIFAR-10, which corroborates our findings from the synthetic data experiments.

\mypar{Random polynomial.} $p(x)$ of degree $D$ for $x \in \mathbb{R}^d$ with sum of coefficient absolute values \rebut{$< 1/D$}. \rebut{ As a result, we ensure that the function is 1-Lipschitz over $[-1,1]^d$.}

\mypar{Random hypercube function.} $f: [-1, 1]^d \rightarrow \mathbb{R}$ which interpolates the indicator functions at each corner with random $\{-1, 1\}$ value at each corner. Concretely, the function is defined as follows: for each corner point $y \in \{-1, 1\}^d$, its indicator function is $I_y(x) = \prod_{i=1}^{d} \frac{1 + y_i x_i}{2}$. Sample random values $v_{y} \in \{-1, 1\}$ with probability $(0,5, 0.5)$ independently for each $y \in \{-1, 1\}^d$, the random hypercube polynomial function is $f(x) = \sum_{y \in \{-1, 1\}^d} v_y I_y(x)$.

For the random polynomial (random hypercube) functions, we use $d = 8$ and $D = 4$ in the this section and provide more results for other parameter settings in the appendix (for example, see Section~\ref{sec:manifold-exps} for when when the target polynomial functions concentrate on a low-dimensional subspace).

\subsection{Results} 

\mypar{MSE for random functions.} Figures \ref{fig:sparse_poly} and \ref{fig:sparse_hypercube} show the scaling behavior of the $\DSM$ and LSH models for a random polynomial function and  hypercube function (lower is better). Sparsity helps for the $\DSM$ and LSH models, both achieving better quality than dense models using the same number of activated units. In Figure~\ref{fig:comparison}, we compare the $\DSM$/LSH models with the random hash sparse models, and we see random hash sparse models underperform dense models, suggesting data-dependent sparsity mechanisms (such as $\DSM$/LSH) are effective ways to utilize sparsity.

\mypar{FLOPs.} To further qualify the efficiency gain of sparse models, we compare the MSE at the same FLOPs for sparse/dense models in Table~\ref{tab:flops}. The first column is the \# FLOPs for the dense model; models in the 3rd and 4th columns use same \# FLOPs but have more parameters (only 50\% or 25\% active). DSM uses only 18k FLOPs and gets smaller MSE than dense model with 73k FLOPs.

\begin{table}[h!]
    \centering
    \vspace{-0.05in}
    \begin{tabular}{|c|c|c|c|} \hline
    FLOPs & eval MSE (dense) & eval MSE (DSM 50\% sparsity) &  eval MSE (DSM 25\% sparsity)\\ \hline
    18432 & 0.01015	& 0.01014 &	{\bf 0.009655}\\ \hline
    36864 &	0.01009	& 0.007438 & {\bf 0.005054}\\ \hline
    73728 &	0.01046 & 0.006115	& {\bf 0.001799}\\ \hline
    \end{tabular}
    \caption{FLOPs and evaluation Mean Squared Error (eval MSE).}
    \vspace{-0.1in}
    \label{tab:flops}
\end{table}

\begin{figure}[h!]
  \centering
  \begin{subfigure}[b]{0.5\linewidth}
    \centering
    \includegraphics[width=\textwidth]{figs/dsm_scaling_sc.pdf}
    \vspace{-0.15in}
    \label{fig:sparse_poly_a}
  \end{subfigure}%
  \begin{subfigure}[b]{0.5\linewidth}
    \centering
    \includegraphics[width=\textwidth]{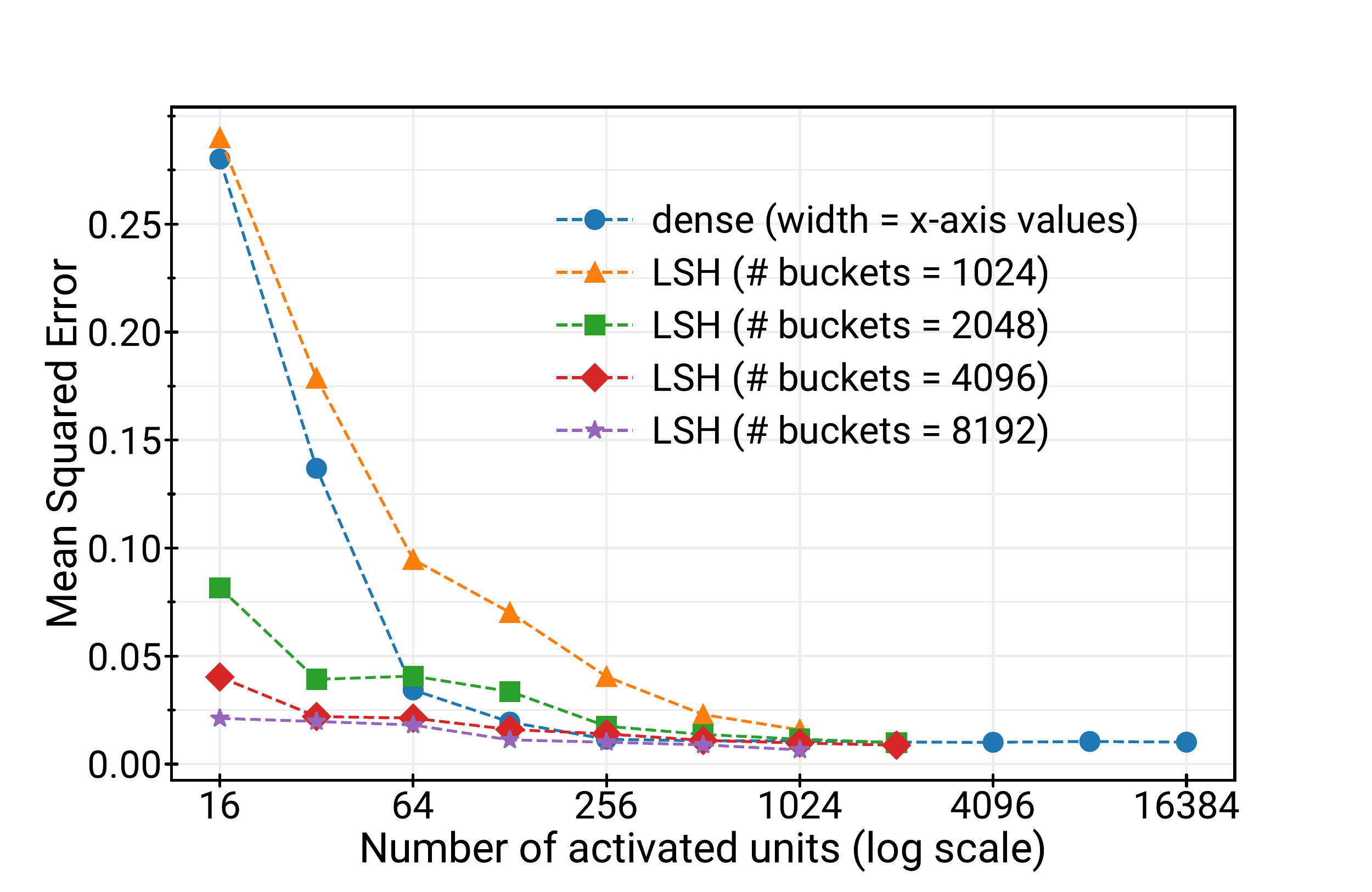}
    \vspace{-0.15in}
    \label{fig:sparse_poly_b}
  \end{subfigure}
  \caption{Scaling behavior of $\DSM$ and LSH models compared with dense models for a degree $4$ random polynomial with input dimension $8$: (a) $\DSM$ outperforms dense model at the same number of activated units and quality improves as total width increases; (b) LSH model outperforms dense model when number of buckets is large ($\geq 2048$) and quality improves as number of buckets increase.}
   \label{fig:sparse_poly}
   \vspace{-0.1in}
\end{figure}
\begin{figure}[h!]
  \centering
  \begin{subfigure}[b]{0.5\linewidth}
    \centering
    \includegraphics[width=\textwidth]{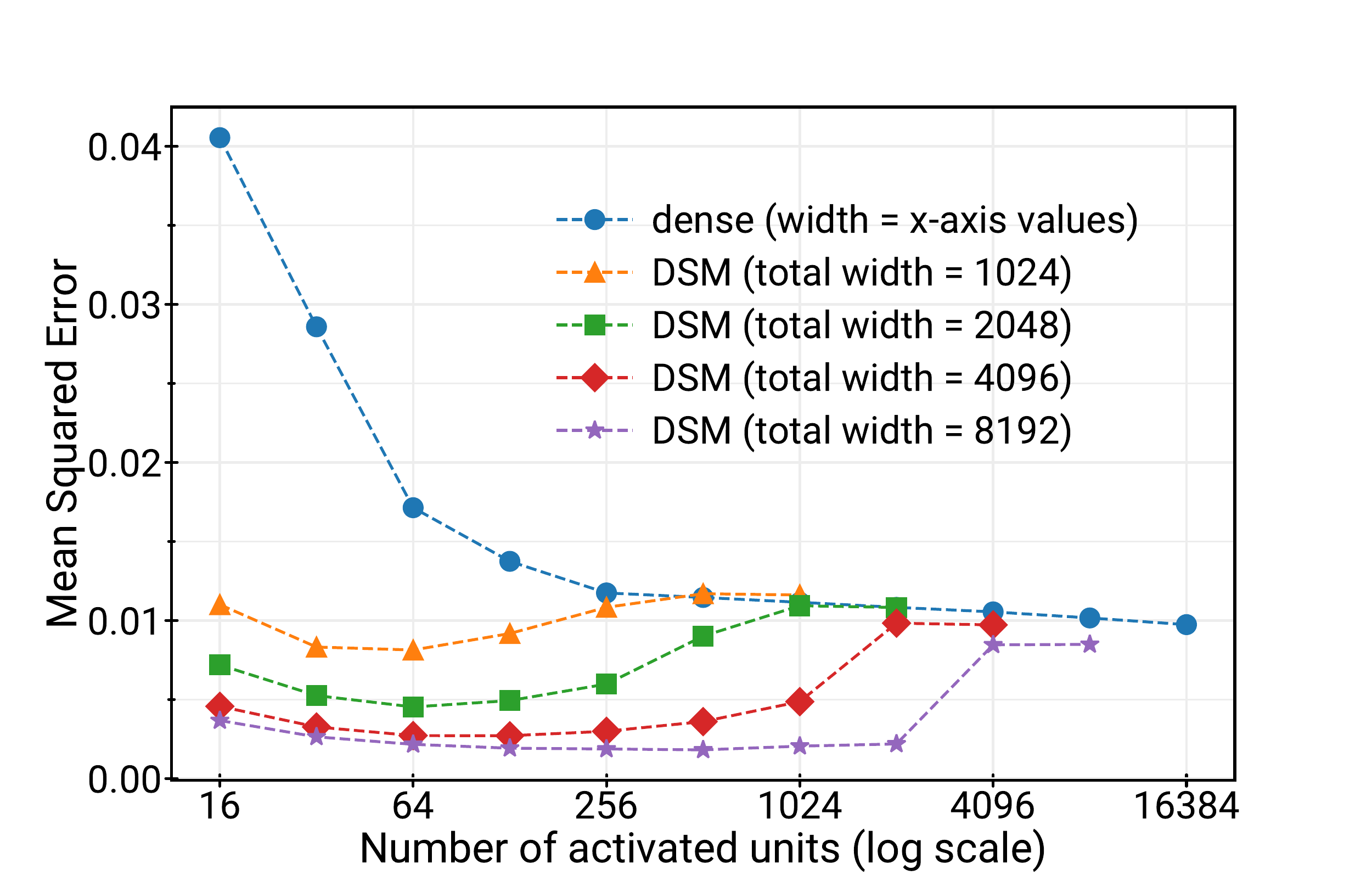}
    \vspace{-0.15in}
    \label{fig:sparse_cube_a}
  \end{subfigure}%
  \begin{subfigure}[b]{0.5\linewidth}
    \centering
    \includegraphics[width=\textwidth]{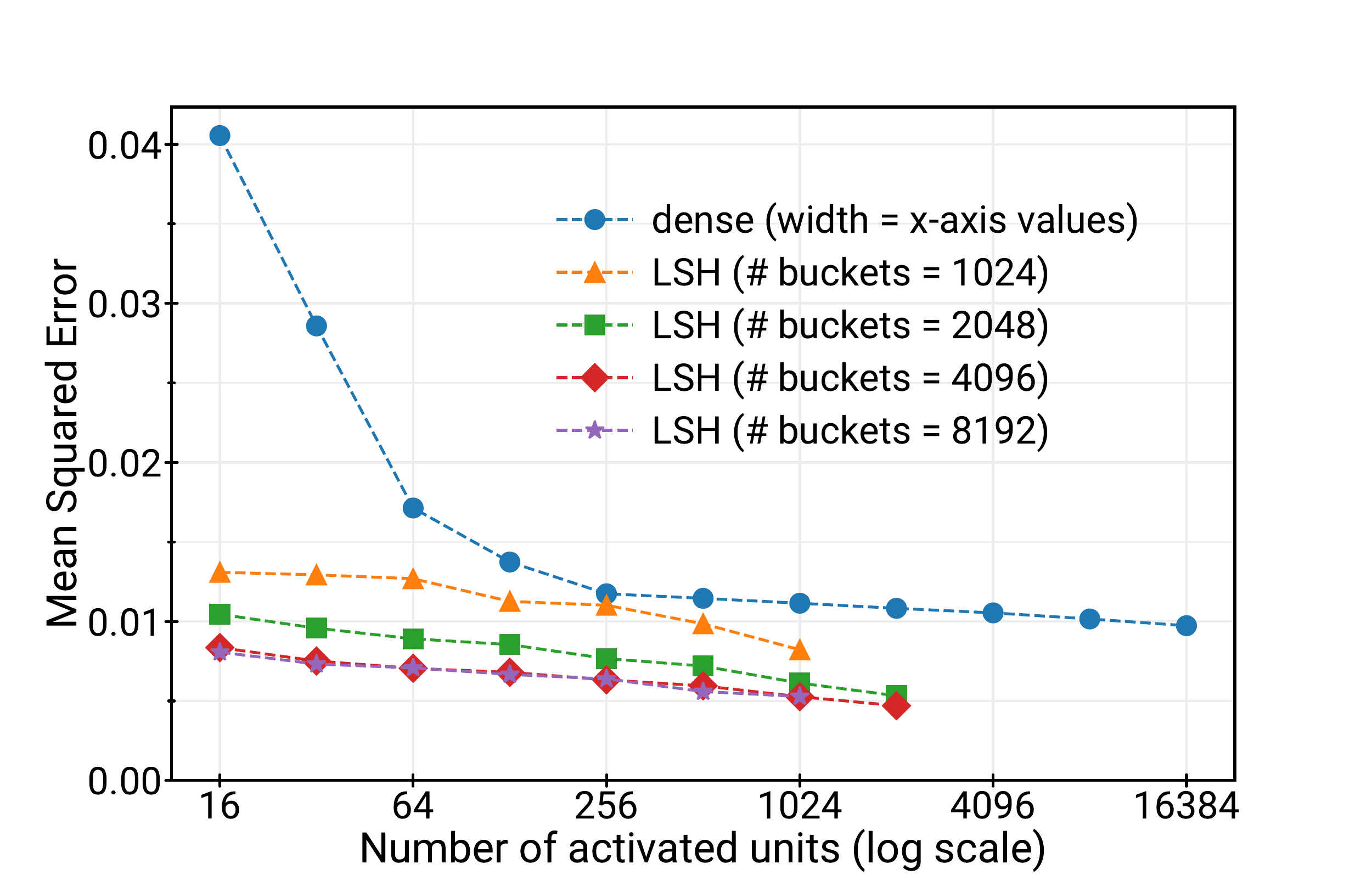}
    \vspace{-0.15in}
    \label{fig:sparse_cube_b}
  \end{subfigure}
  \caption{Scaling behavior of $\DSM$ and LSH models compared with dense models for a random hypercube function with input dimension $8$. Both $\DSM$ and LSH models outperform corresponding dense models with the same number of activated units (note: \# activ.~units $\leq$ width).}
   \label{fig:sparse_hypercube}
   \vspace{-0.1in}
\end{figure}

\mypar{CIFAR-10.} We also compare the scaling behavior of DSM and dense models on CIFAR-10~\cite{cifar}. The baseline model is a CNN with $3$ convolutional layers (followed by max-pooling), a dense layer with varying number of units, and a final dense layer that computes the logits (referred as CNN + dense). For the data-dependent sparse models, we use the same architecture, except we change the penultimate dense layer with a data-dependent sparse layer that choose the Top-$K$ activations (referred as CNN + DSM). Both models are trained with ADAM optimizer for $50$ epochs and evaluated on the test dataset for model accuracy with no data augmentation; see Figure~\ref{fig:cifar10} and Table~\ref{tab:cifar10} for the accuracy versus number of activated units. As with the synthetic datasets,  DSMs outperform dense models at the same number of activated units.



 \subsection{Discussion} 
 Our experimental results
 (e.g., Figures \ref{fig:sparse_poly},  \ref{fig:sparse_hypercube}, and \ref{fig:comparison}) show that sparsely activated models can efficiently approximate both random polynomials and random hypercube functions. Intuitively, the $\DSM$ and $\LSH$ models employ the sparsity as a way to partition the space into nearby input points.  Then, because the target function is Lipschitz, it is easy to provide to local approximation tailored to the specific sub-region of input space. On the other hand, the uniform random hash function performs poorly for these tasks precisely because it does not capture the local smoothness of the target function.
 
 On CIFAR-10, we also see that the $\DSM$ model performs comparably or better than the dense network. In particular,  Figure~\ref{fig:cifar10} shows that the ``CNN + DSM'' approach with total width 1024 improves upon or matches the dense model. In this case, the sparse activation allows the network to classify well while using only a fraction of the number of FLOPs. In Table~\ref{tab:cifar10}, we see that $\DSM$ model outperforms the dense model when we control for the number of activated units in the comparison.
 


\mypar{Limitations.} Our experiments focus on small datasets and 2-layer networks as a way to align with our theoretical results. Prior work on sparsely activated networks has shown success for large-scale NLP and vision tasks. Our experiments complement previous results and justify the $\DSM$ and $\LSH$ models by showing their ability to approximate Lipschitz functions (consistent with our theorems). It would be good to further evaluate the $\DSM$ and $\LSH$ models for large-scale tasks, for example, by using them as inspiration for routing functions in MoE vision transformers, such as V-MoE~\cite{riquelme2021scaling}. It would be interesting to evaluate on larger datasets, such as ImageNet as well. We experimented with a handful of hyperparameter settings, but a full search may lead to different relative behavior between the models (similarly, we only evaluated a few parameter settings for the random functions and the synthetic data generation, which is far from exhaustive).

\begin{figure}[t!]
  \centering
  \begin{subfigure}[b]{0.5\linewidth}
    \centering
    \includegraphics[width=\textwidth]{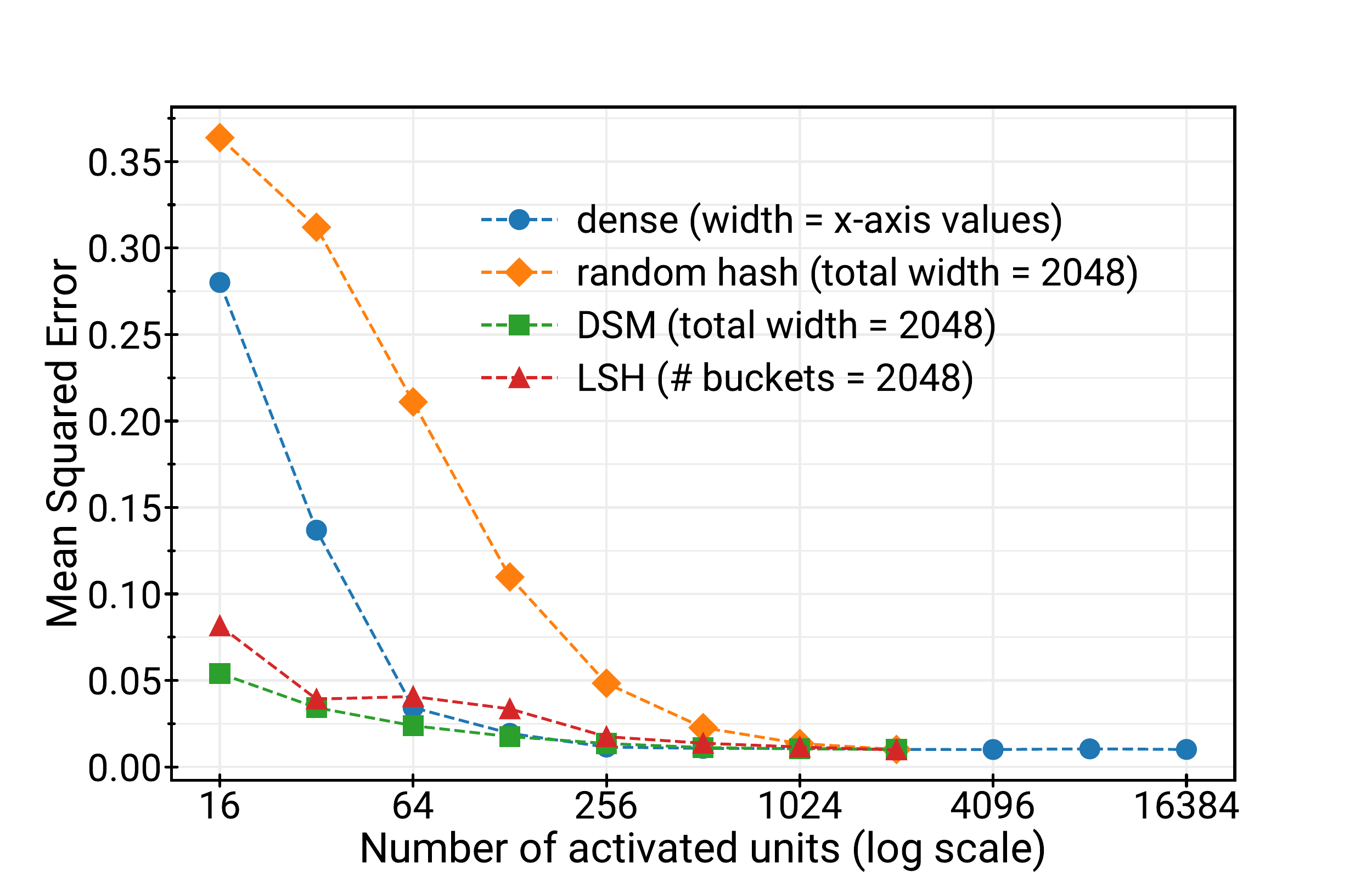}
    \caption{random polynomial}
    \label{fig:sparse_poly_a_hash}
  \end{subfigure}%
  \begin{subfigure}[b]{0.5\linewidth}
    \centering
    \includegraphics[width=\textwidth]{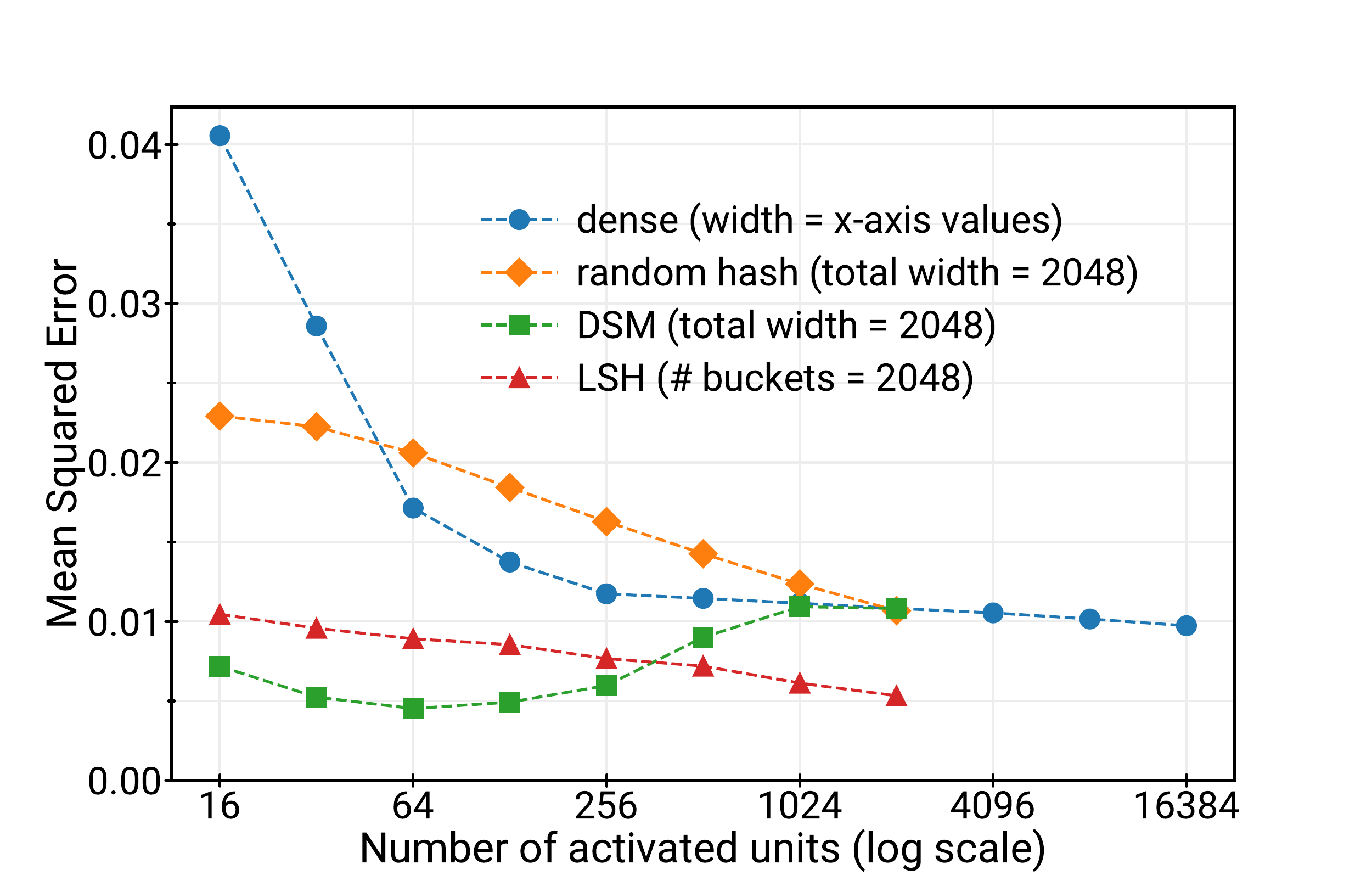}
    \caption{random hypercube function}
    \label{fig:sparse_poly_b_hash}
  \end{subfigure}
  \caption{Scaling behavior of dense, random hash, $\DSM$, and LSH models. $\DSM$ and LSH models outperform dense models, while random hash models underperform dense models with the same number of activated units.}
   \label{fig:comparison}
   \vspace{-0.1in}
\end{figure}

 \begin{figure}[t]
\begin{floatrow}
\ffigbox{%
  \includegraphics[scale=0.25]{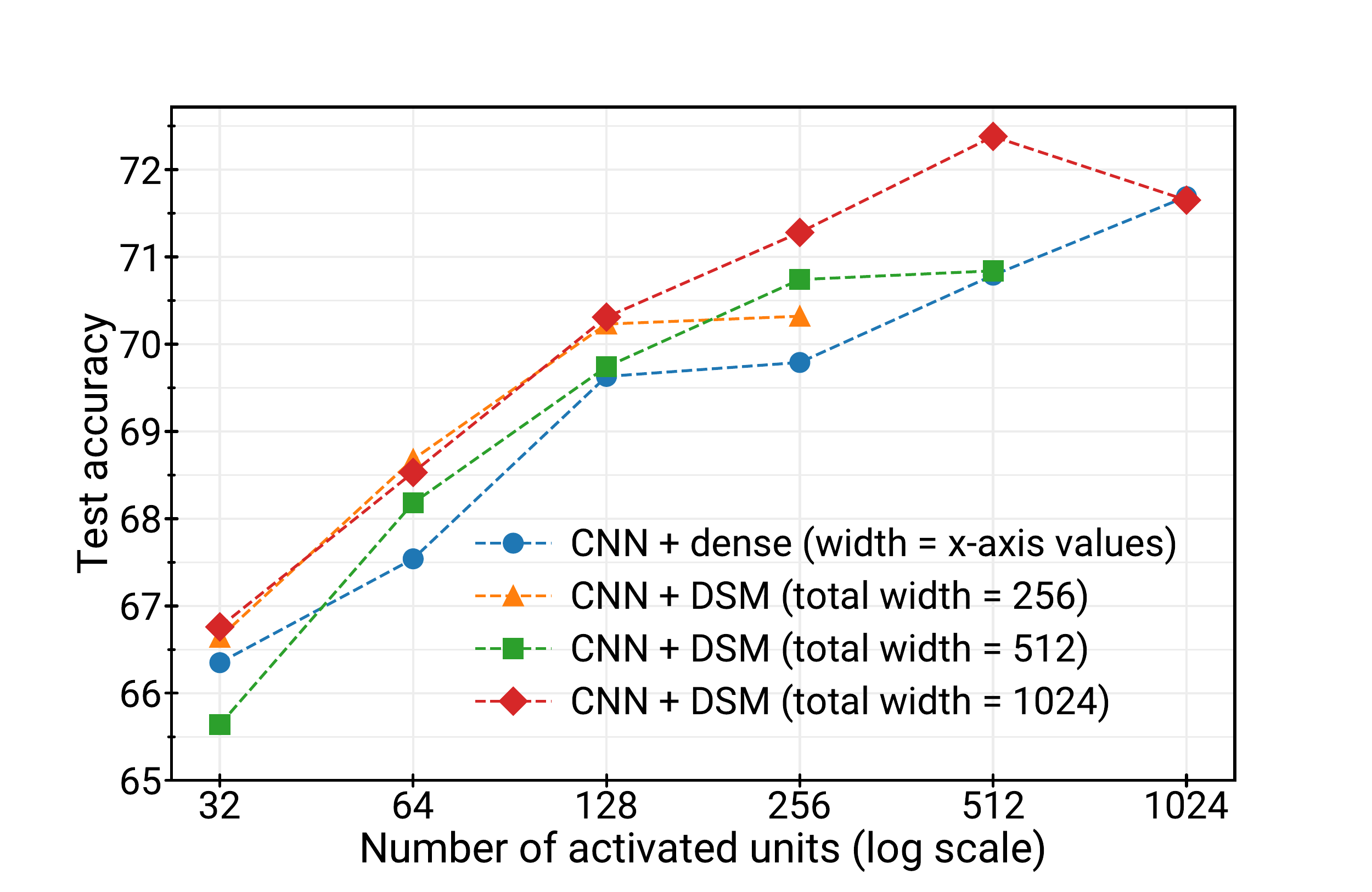}%
}{%
  \caption{Scaling behavior of $\DSM$ compared with dense models on the CIFAR-10 dataset. Similar to Figure~\ref{fig:sparse_poly_a_hash}, $\DSM$ outperforms dense models at the same number of activated units. }%
  \label{fig:cifar10}
}
\capbtabbox{%
 \begin{tabular}{|c|c|c|} \hline
    Model \textbackslash \  \# activated units & 256 &  512 \\ \hline
    Dense	& 69.79 & 70.79 \\ \hline
    $\DSM$ (50\% sparse)	& {\bf 70.74}	 & 71.33 \\ \hline
    $\DSM$ (25\% sparse)	& 69.8	 & {\bf 71.68} \\ \hline
 \end{tabular}
}{%
  \caption{CIFAR-10 test accuracy for dense/$\DSM$ models with the same number of activated units. While not strictly monotonic, wider and sparser models outperform narrow and dense ones.}%
  \label{tab:cifar10}
}
\end{floatrow}
\end{figure}



\section{Conclusion}
\label{sec:conclusion}

We provided the first systematic theoretical treatment of modern sparsely activated networks. To do so, we introduced the $\DSM$ model, which captures the sparsity in Mixture of Experts models, such as Switch and Scaling Transformers. We showed that $\DSM$ can simulate popular architectures as well as LSH-based networks. Then, we exhibited new constructions of sparse networks.
Our use of LSH to build these networks offers a theoretical grounding for sparse networks.
We complemented our theory with experiments, showing that sparse networks can approximate various functions.

For future work, it would be interesting to implement LSH-based networks in transformers for language/vision tasks. A related question is to determine the best way to interpolate in each LSH bucket (e.g., a higher degree polynomial may work better). 
Another question is whether a dense model is more powerful than a sparse model with the same number of total trainable parameters. Theorem~\ref{thm:main-lipschitz-lsh-ub} only says that a sparse model with similar number of parameters as a dense model can more efficiently (fewer FLOPs) represent Lipschitz functions. This does not say \emph{all} functions expressible by a dense model are also expressible by a sparse model. This is non-trivial question as $A^x$ depends on the input (i.e., $\DSM(d,t,t)$ may be more expressive than the dense model with width $t$).
We expect that dense networks can be trained to perform at least as well as sparse networks, \emph{assuming the width is large enough}. The dense networks should optimize the weights in the last layer to approximate the function, but they may not favor certain neurons depending on the input.

\bibliography{refs}
\bibliographystyle{plain}
\newpage
\appendix

\iftrue

\section{Proofs of Main Upper and Lower Bounds}

\subsection{Proof of the Lipschitz Upper Bound}

\begin{proof}[Proof of Theorem~\ref{thm:main-lipschitz-lsh-ub}.]
We use the Euclidean-$\LSH$ construction of Lemma~\ref{lem:lsh-buckets-size} with parameter $\epsilon/L$. In any sub-region of the $k$-dimensional subspace that has a small diameter, the Lipschitz nature of the function together with Lemma~\ref{lem:lsh-buckets-size} will imply that we can approximate it by just a constant and incur only $\epsilon$ error in $\ell_{\infty}$. In particular, given a point $x_1$ belonging to an LSH bucket, we can set $\hat{f}(x) = f(x_1)$ everywhere in that bucket. For any $x_2$ also mapping to the same bucket, from Lemma~\ref{lem:lsh-buckets-size}, we have that $\|x_1 - x_2\|_2 \le \epsilon/L$. Since $f$ is L-Lipschitz,
\begin{align}
    |\hat{f}(x_2) - f(x_2)| = |f(x_1) - f(x_2)| \le L\|x_1 - x_2\|_2 \le \epsilon.
\end{align}

Next we look at how many samples we need to obtain the guarantee $\|\hat{f} - f\|_{\infty} \le \epsilon$. A rare scenario that we have to deal with for the sake of completeness is when there exist buckets of such small volume that no training data point has mapped to them and consequently we don't learn any values in those buckets.
At test time, if we encounter this rare scenario of mapping to a bucket with no value learnt in it, we simply run an approximate nearest neighbor search among the train points. For our prediction, we use the bucket value associated with the bucket that the approximate nearest neighbor maps to. To control the error when doing such a procedure, we take enough samples to approximately form an $\epsilon/\Gamma L$ cover of $\Gamma$ for a large enough constant $\Gamma$.
The size of an $\epsilon/\Gamma L$ cover of $\Gamma$ is $O((2\Gamma L\sqrt{d}/\epsilon)^k)$. This implies that, via a coupon collector argument, when the input distribution is uniform over the region $\Gamma$,
$O(k(2\Gamma  L\sqrt{d}/\epsilon)^k\log(2\Gamma  L\sqrt{d}/\epsilon))$ samples will ensure that with very high probability, for every test point $x$ there exists a train example $x_i$ such that $\|x-x_i\|_2 \le 2\epsilon/\Gamma L$. The test error is $|f(x) - \hat{f}(x_i)| \le |f(x) - f(x_i)| + |f(x_i) - \hat{f}(x_i)| = O(\epsilon)$. Computing the exact nearest neighbor is a slow process. Instead we can compute the approximate nearest neighbor using LSH very quickly. We lose another $O(\epsilon)$ error due to this approximation. Choosing $\Gamma$ appropriately we can make the final error bound exactly $\epsilon$.
This leads to our stated sample complexity bound.

This implies that
$  
\|f - \hat{f}\|_{\infty} \le \epsilon.
$
Hence using an Euclidean $\LSH$ with $O(k)$ hyperplanes we can learn an $\epsilon$-approximation to $f$. The time to compute $\hat{f}(x)$ for a new example is the time required to compute the bucket id where it maps to. Since there are $k$ hyperplanes and our input is $d$-dimensional, computing the projections of $x$ on the $k$ hyperplanes takes $O(dk)$ time. Then we need to perform a division by the width parameter $\epsilon/L$. This takes time equal to the number of bits required to represent $L/\epsilon$. Hence the total time taken would be $O(dk\log(L/\epsilon))$.
\end{proof}

The above theorem uses a lemma about Euclidean LSH, which we prove next.

\begin{proof}[Proof of Lemma~\ref{lem:lsh-buckets-size}.]
Let $K = Ck$. Let the random hyperplanes chosen by the Euclidean LSH be $a_1, \ldots, a_K$. Let the width parameter used by the LSH be $\epsilon_{LSH}$. The value of $\epsilon_{LSH}$ we choose will be determined later. Since the distribution of entries is spherically symmetric, the projection of the vectors onto the $k$-dimensional subspace will also form a Euclidean-$\LSH$ model. Henceforth in our analysis we can assume that all our inputs are projected onto the $k$-dimensional space $\Gamma$ and that we are performing LSH in a $k$-dimensional space instead of a $d$-dimensional one. Let $A = [a_1, \ldots, a_K]^\top$ be the matrix whose columns are the vectors perpendicular to the hyperplanes chosen by the LSH. Note that $A \in \mathbb{R}^{K \times k}$. Then we have, from tail properties of the smallest singular value distribution of Gaussian random matrices (e.g. see \cite{chen2005condition,rudelson2009smallest}), for a large enough constant $c$,
\begin{align}
    \Pr[\sigma_{\min}(A) \ge c\sqrt{k}] \ge 9/10. \label{eq:smallest-singular-val}
\end{align}
For two points $x_1,x_2 \in \Gamma$ to map to the same LSH bucket, $\|A(x_1-x_2)\|_{\infty} \le \epsilon_{LSH}$. This implies that $\|A(x_1 - x_2)\|_2 \le \epsilon_{LSH}\sqrt{k}$. Together with \eqref{eq:smallest-singular-val}, this implies that $\|x_1 - x_2\|_2 \le \epsilon_{LSH}/c$ with probability $\ge 9/10$.
At the same time, since we $x_1, x_2 \in [-1,1]^d$, the maximum distance along any direction is at most the length of any diagonal, which is $2\sqrt{d}$. Moreover, along any hyperplane direction sampled by the LSH, we grid using a width $\epsilon_{LSH}$. Since the total number of hyperplanes is $Ck$ the maximum number of LSH buckets possible is
$
 \left(\frac{2\sqrt{d}}{\epsilon_{LSH}}\right)^{O(k)}.
$
We set $\epsilon_{LSH} = c\epsilon$. Then, the diameter of any bucket $\le \epsilon_{LSH}/c = \epsilon$. The upper bound on the maximum number of LSH buckets required to cover the region $\Gamma$ also follows.
\end{proof}

\subsection{Proof of the Dense Lower Bound}

\begin{proof}[Proof of Theorem~\ref{thm:dense-lb}.]
Assuming $B$ bits per parameter, in our dense layer model we have $2^{Bt}$ distinct possible configurations. We lower bound the width $t$ by constructing a class of functions $\mathcal{F}$ defined on a $k$-dimensional subspace within $[-1, 1]^d$ such that three properties simultaneously hold:
\begin{enumerate}
    \item each $f \in \mathcal{F}$ is $L$-Lipschitz,
    \item the number of functions in $\mathcal{F}$ is at least $\Omega(2^{(2\sqrt{d}L/C\epsilon)^k})$
    \item for $f_1 \neq f_2 \in \mathcal{F}$, we have
    $\|f_1 - f_2\|_{\infty} > \epsilon$.
\end{enumerate} 
These three properties together will imply that $t \ge \frac{1}{B}\dot
(2\sqrt{d}/C\epsilon)^k$ as otherwise by there would have to be two functions $f_1 \ne f_2 \in \mathcal{F}$ that are approximated simultaneously by the same dense network, which is impossible since $\|f_1 - f_2\|_{\infty} > \epsilon$.
We construct $\mathcal{F}$ as follows. Given the $d$-dimensional cube $[-1,1]^d$, we pick a subset of $k$ diagonals of the cube such that they are linearly independent. We consider the $k$-dimensional region defined by the intersection of the subspace generated by these diagonals and the cube $[-1,1]^d$. Denote the region we obtain by $G$. Let $e_1,\ldots, e_k$ form an orthonormal basis for the subspace $G$ lies in. We grid $G$ into $k$-dimensional cubes of side length $2\epsilon/L$ aligned along its bases $\{e_i\}_{i=1}^k$. For the center of every cube we pick a random assignment from $\{+\epsilon, -\epsilon\}$. Then we interpolate the function everywhere in $G$ such that (i) it satisfies the assigned values at the centers of the cubes and (ii) its value decreases linearly to 0 with radial distance from the center. That is, given the set of cube centers $V$
$$f(x) = \sum_{v \in V} \max(0, f(v)-L\sgn(f(v))\|x-v\|_2)$$
To understand the Lipschitz properties of such an interpolation, note that the slope at any given point in $G$ is either $0$ or $L$, which bounds the Lipschitz constant by $L$. The total number of cubes that lie within $G$ is at least $(\sqrt{d}L/C\epsilon)^k$ for some constant $C$ and hence $\mathcal{F}$ contains a total of $(2)^{(\sqrt{d}L/C\epsilon)^k}$ functions. Moreover, given any $f_1, f_2 \in \mathcal{F}$ such that $f_1 \ne f_2$, there exists a cube center where their values differ by $2\epsilon$ giving us the third desired property as well. Consequently, we get that to attain $\epsilon$-uniform error successfully on $\mathcal{F}$ we need $$2^{Bt} \ge 2^{(\sqrt{d}L/(C\epsilon))^k},$$ which implies that $t = \Omega((\sqrt{d}L)^k/(C\epsilon)^k)$.
\end{proof}

\section{$\LSH$ Models Can Also Learn Lipschitz Functions on $k$-Manifolds}
\label{sec:manifolds-apdx}
A $k$-dimensional manifold (referred to as a $k$-manifold) can loosely be thought of as a $k$-dimensional surface living in a higher dimensional space. For example the surface of a sphere in 3-dimensions is a 2-dimensional manifold.
We consider $k$-manifolds in $\mathbb{R}^d$ that are homeomorphic to a $k$-dimensional subspace in $\mathbb{R}^d$. We assume that our $k$-dimensional manifold $M_k$ is given by a transform $f:R^k \rightarrow R^k$ applied on $k$-dimensional subspace of $\mathbb{R}^d$ $L_k$. To control the amount of distortion that can occur when going from $L_k$ to $M_k$, the Jacobian of $f$ is assumed to have a constant condition number for all $x \in L_k$. We now state our main upper bound for manifolds, showing that LSH models can adapt and perform well even with non-linear manifolds of a bounded distortion from a linear subspace.



\begin{theorem}\label{thm:main-lipschitz-ub-manifolds}
For any $f: [-1, 1]^d \to \mathbb{R}$ that is $1$-Lipschitz, and for an input distribution $\mathcal{D}$ that is uniform on a $k$-manifold in $[-1,1]^d$, an $\LSH$ model can learn $f$ to $\epsilon$-uniform error with $O(k\sqrt{dk}^k\log(\sqrt{dk}/\epsilon)/\epsilon^k)$ samples using a hash table of size $O(\sqrt{dk}^k/\epsilon^k)$ with probability $\ge 0.8$.
The total time required for a forward pass on a new test sample is $O(dk\log(1/\epsilon))$.
\end{theorem}
\begin{proof}
The main idea of the proof is to follow similar arguments from Theorem~\ref{thm:main-lipschitz-lsh-ub} on the subspace $L_k$ and try to bound the amount of distortion the arguments face when mapped to the manifold $M_k$. Since we are no longer dealing with a subspace (linear manifold), the argument that an LSH in $d$-dimensions can be viewed as an equivalent LSH in $k$-dimensions does not hold. We use Euclidean-$\LSH$ models with $O(d)$ hyperplanes. Furthermore, we will use multiple $\LSH$ models each defined using $O(d)$ hyperplanes. The main challenge in the proof is to show that the total number of buckets used in approximating $f$ do not grow exponentially in $d$, which is a possibility now as we use $O(d)$ hyperplanes.

\begin{lemma}
For any $x \in \mathbb{R}^d$, a $d$-dimensional sphere of radius $O(\eps/d)$ centered at $x$ is fully contained in the bucket where the center of the sphere maps to with probability $\ge 0.9$.
\end{lemma}
\begin{proof}
Along any hyperplane direction the gap between parallel hyperplanes is $\eps$. Since any point is randomly shifted before being mapped to a bucket we get that with probability $1- O(1/d)$, $x$ is more than $\Omega(1/d)$ away from each of the two parallel hyperplanes on either side. So with probability $(1- O(1/d))^{O(d)} = \Omega(1) $ the entire sphere is contained inside the LSH bucket $x$ maps to.
\end{proof}

\begin{lemma}
Using $O(k \log d)$ Euclidean-$\LSH$ functions, we get that every $x \in L_k$, there exists a bucket in at least one of the $O(k\log d)$ buckets $x$ gets mapped to such that the entire $k$-dimensional sphere of radius $O(\eps/d)$ centered at $x$ is contained within the bucket.
\end{lemma}
\begin{proof}
We use a covering number argument. The maximum volume of a $k$-dimensional subspace within $[-1, 1]^d$ is $(2\sqrt{d})^k$. We cover this entire volume using spheres of radius $\eps/d$. The total number of spheres required to do this are $O((2d\sqrt{d})^k/\eps^k)$.
We now do a union bound over all the sphere centers in our cover above. For a single sphere, the probability that it does not go intact into a bucket in any of the $O(k\log d)$ LSH functions is $d^{-\Omega(k)}$. By a union bound, we can bound the probability that there exists a sphere center that does not go into a bucket to be $d^{-\Omega(k)}$. Hence the Lemma statement holds with exceedingly large probability of $1-d^{\Omega(k)}$.
\end{proof}

Now, we only include buckets with volume at least $(\Omega(\eps/(d\sqrt k)))^{k}$. We can do this procedure using approximate support estimation algorithms~\cite{valiant2011estimating}. This takes time and sample complexity $S/\log S$ where $S$ is the size of the support. With constant probability all points in $L_k$ are mapped to some such high volume bucket in at least one of the LSH functions. The total number of buckets with this minimum volume is at most $(O((d^2\sqrt k)/\eps))^{k}$, which is also an upper bound on the sample complexity and running time of the support estimation procedure.
Now, we lift all the above results when we go to $M_k$ from $L_k$.
Since the Jacobian of the manifold map $f$ has a constant condition number, its determinant is at most $\mathrm{exp}(k)$; so the volume of any region in $L_k$ changes by at most an $\mathrm{exp}(\pm O(k))$ multiplicative factor when it goes to $M_k$. So all volume arguments in the previous proofs hold with multiplicative factors $\mathrm{exp}(\pm O(k))$.
This concludes our proof.
\end{proof}


\section{Lower Bound for Analytic Functions}

\begin{figure}[ht!]
    \centering
    \begin{subfigure}[b]{0.4\linewidth}
        \centering
        \includegraphics[width=\textwidth]{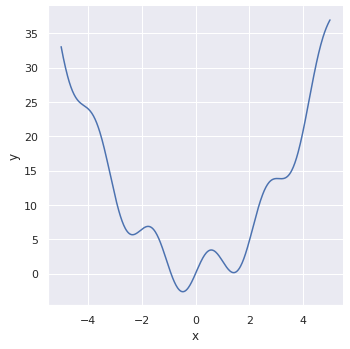}
        \label{fig:analytic-vs-non-analytic-a}
    \end{subfigure}
    \begin{subfigure}[b]{0.4\linewidth}
        \centering
        \includegraphics[width=\textwidth]{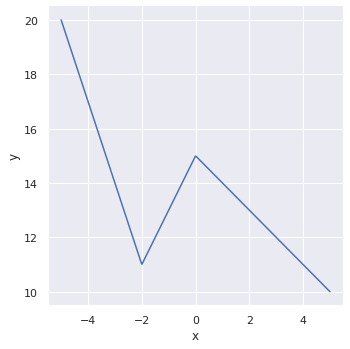}
        \label{fig:analytic-vs-non-analytic-b}
    \end{subfigure}
    \caption{On the left we have an example of an analytic function. The function on the right is not analytic. Both functions have a bounded Lipschitz constant 
    }
    \label{fig:analytic-vs-non-analytic}
\end{figure}

The functions described in the lower bound presented earlier are continuous but not differentiable everywhere as they are piecewise linear functions. In Theorem~\ref{thm:dense-lb-analytical} we show that we can make the lower bound stronger by providing a construction of $L$-Lipschitz analytic functions (which are differentiable everywhere). See Figure~\ref{fig:analytic-vs-non-analytic} for an example of analytic vs non-analytic functions.

\begin{theorem}
\label{thm:dense-lb-analytical}
A dense model of width $t$ with a random bottom layer requires 
$$t = \Omega\left(\frac{2^{k^2/2}(LC_1)^k}{(\sqrt{k}\pi\epsilon)^k}\right),$$
where $C_1$ is a large enough constant,
to learn $L$-Lipschitz analytic functions on $[-1,1]^d$ to $\ell_{\infty}$ error $\epsilon$ when the inputs are sampled uniformly over a unknown $k$-dimensional subspace of $\mathbb{R}^d \cap [-1,1]^d$. Moreover, the number of samples required to learn the above class of functions is
$$\Omega\left(t\log t\right),$$
where $t = \Omega\left(\frac{2^{k^2/2}(LC_1)^k}{(\sqrt{k}\pi\epsilon)^k}\right)$.
\end{theorem}

\label{sec:theory-apdx}
\begin{proof}[Proof of Theorem~\ref{thm:dense-lb-analytical}]
We construct a family $\mathcal{F}$ of analytic functions that are $L$-Lipschitz and are described using the Fourier basis functions. 
Each $f \in \mathcal{F}$ will be of the form 
$$ f(x) = \sum_{n_1=0}^{\infty}\sum_{n_2=0}^{\infty}\cdots\sum_{n_k=0}^{\infty} a_{n_1n_2\ldots n_k}\exp \left(i\pi n^\top x \right),$$
for $x \in [-1,1]^k$. We pick a small value of $0 < \epsilon_1 < 1$. We assume $1/\epsilon_1$ is an integer for convenience. If it is not, we can simply take $\lceil 1/\epsilon_1 \rceil$ instead. For a set of integers $(n_1, n_2, \ldots, n_k) \in [1/\epsilon_1]^k$, let $\eta_{n_1 n_2 \ldots n_k} \in \{ \pm \}$. We use $\eta_n$ as a shorthand when it is not ambiguous. The family $\mathcal{F}$ is defined as the set of functions $f$ below
\begin{align}
    f(x) = \sum_{n_1, \ldots, n_k = 0}^{1/\epsilon_1} \eta_{n} \epsilon_1^{\alpha} \left( \exp(i\pi n^\top x) + \exp(i\pi n^\top x) \right),
\end{align}
where each $\eta_{n}$ is chosen to be either $\pm L/(C\sqrt{k}\pi)$ for a large enough constant $C$ and $\alpha$ will be determined later. There are $(1/\epsilon_1)^k$ Fourier bases in each $f$ and the coefficient of each is set to be $\pm L\epsilon_1^{\alpha}/(C\sqrt{k}\pi)$. Hence we have
\begin{align}
    |\mathcal{F}| = 2^{\left( (1/\epsilon_1)^k\right)}.
\end{align}
Next we argue that a larger than $0.9$ fraction of the functions in $\mathcal{F}$ are $L$-Lipschitz.
We have,
\begin{align}
    &\nabla f(x) = \sum_{n_1, \ldots, n_k = 0}^{1/\epsilon_1} \eta_n \epsilon_1^{\alpha} i\pi(\exp(i\pi n^\top x) - \exp(i\pi n^\top x)) n \notag\\
    &~~~~~~~~~~~= \sum_{n_1, \ldots, n_k = 0}^{1/\epsilon_1} -2\eta_n\pi\sin(\pi n^\top x)\epsilon_1^{\alpha}n \\
    &\implies \E\left[ \nabla f(x) \right] = 0,
\end{align}
where the last expectation is over the uniform measure over functions in $\mathcal{F}$.
To get a bound on $\|\nabla f(x)\|_2$ we bound each $(\nabla f(x))_i$ with high probability.
Each $(\nabla f(x))_i$ is a sum of $(1/\epsilon_1)^k$ independent random variables, namely $\eta_n$. We saw above that $\E[(\nabla f(x))_i] = 0$. To bound $|(\nabla f(x))_i|$ with high probability we will use McDiarmid's inequality. 
An upper bound on how much the value of $(\nabla f(x))_i$ can change when any one $\eta_n$ flips in value is computed as $4\epsilon_1^{(\alpha - 1)}L/C\sqrt{k}$.
Then, an application of McDiarmid's concentration inequality gives us that,
\begin{align}
    \Pr \left[|(\nabla f(x))_i| > t \right] \le 2\exp\left( \frac{-t^2kC^2\epsilon_1^{(k+2-2\alpha)}}{16L^2}\right), \notag\\
    \implies |(\nabla f(x))_i| \le \frac{L}{\sqrt{k}\epsilon_1^{(k+2-2\alpha)/2}}
\end{align}
with probability $\ge 0.9$ for a large enough constant $C$.
This implies that
\begin{align}
    \|\nabla f(x)\|_2 \le  \frac{L}{\epsilon_1^{(k+2-2\alpha)/2}} \label{eq:lipschitzness-prop-lb}
\end{align}
with probability $\ge 0.9$ for a randomly sampled $f \in \mathcal{F}$.
Now, let $\eta_f$ denote the vector of $\eta_n$ values in sequence for any $f$. Using McDiarmid's (or Hoeffding's) concentration bound again, we also get that, with probability $\ge 0.9$, the Hamming distance between $\eta_{f_1}$ and $\eta_{f_2}$ for two $f$ randomly sampled from $\mathcal{F}$ is at least $c(1/\epsilon_1)^k$ for a small enough constant $c < 1$. This implies that for randomly sampled $f_1, f_2$,
\begin{align}
    &f_1(x) - f_2(x) \notag\\
    &= \sum_{n_1, \ldots, n_k = 0}^{1/\epsilon_1} 2\eta'_{n} \epsilon_1^{\alpha} \left( \exp(i\pi n^\top x) + \exp(i\pi n^\top x) \right),
\end{align}
where $\eta'_n$ is non-zero for at least $c(1/\epsilon_1)^k$ of the terms from the above argument about the Hamming distance. Parseval's identity then implies that
\begin{align}
    &\frac{1}{2^k} \int_{-1}^1 \ldots \int_{-1}^1 (f_1(x) - f_2(x))^2 dx_1 \ldots dx_k \notag\\
    &\ge 4L^2\epsilon_1^{2\alpha}c\frac{1}{\epsilon_1^kC^2k\pi^2} \notag\\
    &\implies \|f_1 - f_2\|_{\infty} \ge \frac{2^{(k/2+1)}L\sqrt{c}\epsilon_1^{(\alpha-k/2)}}{C\sqrt{k}\pi}. \label{eq:parseval-prop-lb}
\end{align}
Finally we note that by union bound, at least a $0.8$ fraction of the functions in $\mathcal{F}$ satisfy both our Lipschitzness property \eqref{eq:lipschitzness-prop-lb} and \eqref{eq:parseval-prop-lb} simultaneously.
Setting $\alpha = k/2 + 1$ and $\epsilon_1 = \frac{C\sqrt{k}\pi\epsilon}{L\sqrt{c}2^{k/2}}$ we get that to achieve a strictly smaller error than $ 2\epsilon$ in the $\|.\|_{\infty}$ sense, one requires a dense model with a width of
$$\Omega\left(2^{k^2/2}\left(\frac{LC_1}{\sqrt{k}\pi\epsilon}\right)^k\right).$$
\end{proof}


\section{Experiment details}
\label{sec:experimental-details-apdx}
\subsection{Details of learning random functions}
In Section \ref{sec:experiments}, we demonstrated experiments for randomly generated polynomial/hypercube functions. Here we present the details for the experiment settings. 

\textbf{Random function generation.} For the random polynomial functions, we randomly generate coefficients of the monomials by sampling from a uniform distribution $\mathcal{U}([-1, 1])$  and scale the coefficients so that their absolute values sum up to $1.0$ (this is to ensure the Lipschitz constant of the generated function is bounded by a constant independent of dimension and degree of the polynomial). For the random hypercube function, we sample values of the function at each corner independently from a uniform distribution on ${-1, 1}$, and interpolate using the indicator functions.

\textbf{Train/Test dataset generation.} For a given target function $f$ (polynomial or hypercube), we sample independently from $\mathcal{U}\left( [-1, 1]^d \right)$ (where $d$ is the input dimension) to generate the input features $x$ and compute target value $y = f(x)$. The train dataset contains $2^{16}$ $(x, y)$ pairs and the test dataset contains $2^{14}$ $(x, y)$ pairs.

\textbf{Training setting.} All the models in Section \ref{sec:experiments} are trained for $50$ epochs using the RMSProp \cite{rmsprop} optimizer with a learning rate of $10^{-5}$. For the one dimension example in Section \ref{sec:intro}, the model is trained for $200$ epochs using the RMSProp optimizer with a learning rate of $5 \times 10^{-6}$.

\textbf{Random hash sparse model.} We discussed the design of DSM and LSH models in Section \ref{sec:prelims}. Here we present the details of the random hash model, where the sparsity pattern is determined by a random hash of the input data (i.e. the same input data would always have the same sparsity pattern). The following code snippet shows the generation of a random mask that only depends on the input data using TensorFlow 2.x.

\begin{verbatim}
import tensorflow as tf

# seed: a fixed random seed
# inputs: the input tensor
# mask_dim: size of the masked tensor
# num_buckets: a large integer
# k: the dimension after masking

input_dim = inputs.shape[-1]
if input_dim != mask_dim:
  proj = tf.random.stateless_normal(
    shape=(input_dim, mask_dim),
    seed=seed)
  inputs = tf.einsum(
    '...i,io->...o', inputs, proj)
hs = tf.strings.to_hash_bucket_fast(
  tf.strings.as_string(inputs),
  num_buckets=num_buckets)
top_k_hash = tf.expand_dims(
  tf.nn.top_k(hs, k).values[..., -1],
  axis=-1)
mask = hs >= top_k_hash
\end{verbatim}

\subsection{Learning high-dimensional random polynomials lying in a low-dimensional manifold}
\label{sec:manifold-exps}
We present experiment results for learning random polynomial target functions with low intrinsic dimensions. To be precise, the target polynomial is $p(A x)$, where $p$ is a polynomial of degree $D$ with sum of coefficient absolute value $< 1/D$, $x \in \mathbb{R}^d$ is a vector, $A \in \mathbb{R}^{k \times d}$ is a matrix with random orthonormal rows, and $d > k$. Note now the intrinsic dimension of the domain is $k$, while the inputs $x$ has higher dimension $d$. In Figure~\ref{fig:poly_proj}, we compare the mean squared loss for dense models and $\DSM$s for $d = 64$, $k=8$, and $D=4$. We observe similar behavior as Figure~\ref{fig:sparse_poly}, where the input dimension is the same as the intrinsic dimension. This validates our analysis in Section~\ref{sec:model-results}.

\begin{figure}[h!]
  \centering
  \includegraphics[scale=0.35]{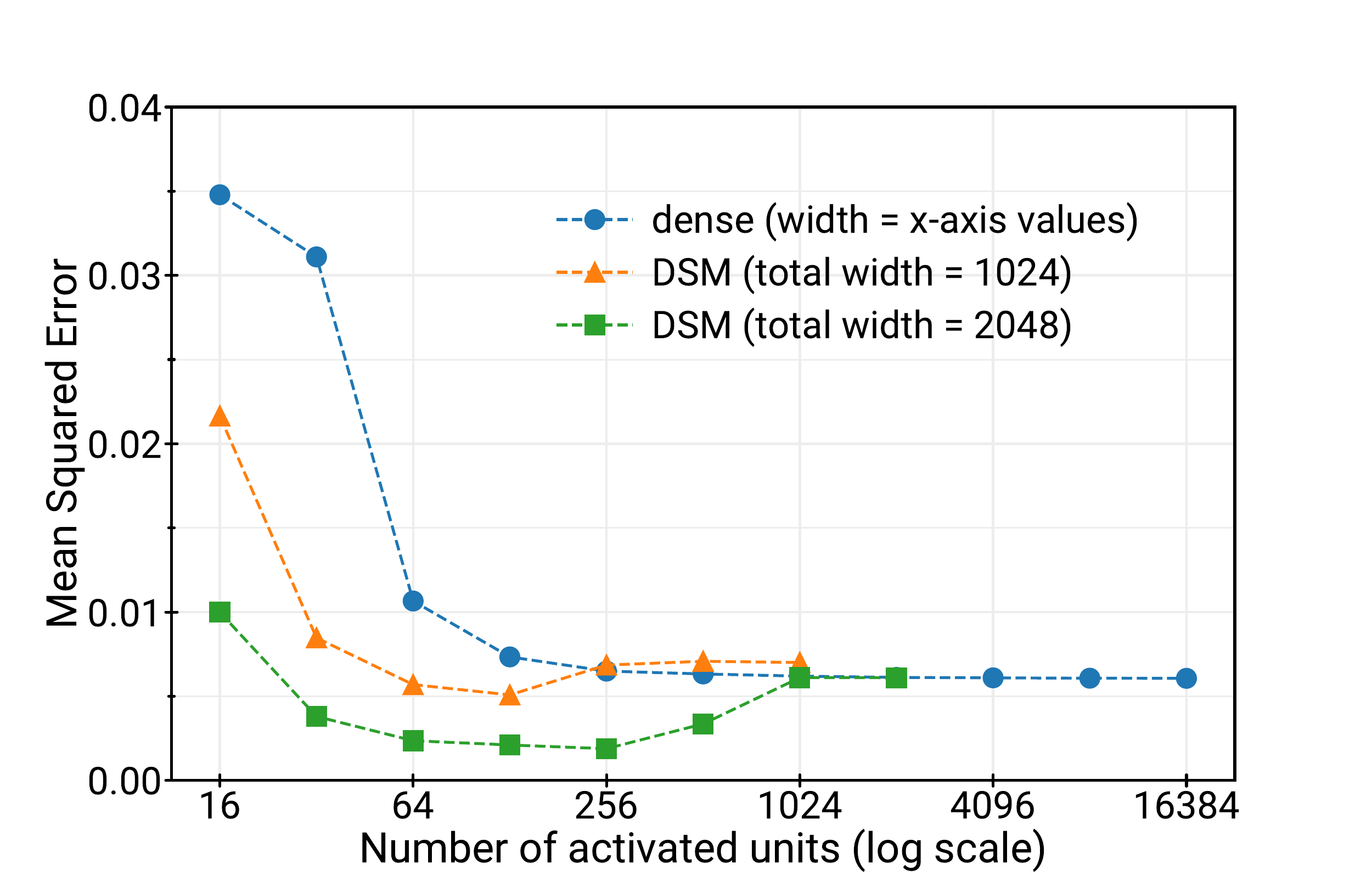}
  \caption{Scaling behavior of $\DSM$ compared with dense models for a random polynomial with low intrinsic dimensionality. Similar to Figure~\ref{fig:sparse_poly}, $\DSM$ outperforms dense models at the same number of activated units. }
  \label{fig:poly_proj}
  \vspace{-0.1in}
\end{figure}





\fi

\end{document}